\def\year{2022}\relax
\documentclass[letterpaper]{article} 
\usepackage{aaai22}  
\usepackage{times}  
\usepackage{helvet}  
\usepackage{courier}  
\usepackage[hyphens]{url}  
\usepackage{graphicx} 
\def\UrlFont{\rm}  
\usepackage{natbib}  
\usepackage{caption} 
\DeclareCaptionStyle{ruled}{labelfont=normalfont,labelsep=colon,strut=off} 
\usepackage{algorithm}
\usepackage{amsmath,amssymb,amsthm,bbm,mathtools,bm}
\usepackage{algorithm,algorithmicx,algpseudocode}

\usepackage{mathtools}
\usepackage{amssymb}
\usepackage{soul}
\usepackage{amsmath}
\newcommand{\norm}[1]{\left\lVert#1\right\rVert}

\usepackage{multicol}
\usepackage{multirow}
\usepackage{booktabs}
\usepackage{tabularx}
\usepackage{subfigure}

\newtheorem{definition}{Definition}[]

\newtheorem{lemma}{Lemma}[]

\newtheorem{property}{Property}[]

\usepackage{newfloat}
\usepackage{listings}
\floatstyle{ruled}
\newfloat{listing}{tb}{lst}{}
\floatname{listing}{Listing}

\title{Learnable Spectral Wavelets on Dynamic Graphs to Capture Global Interactions}
\author{
    Anson Bastos,\textsuperscript{\rm 1}
   Abhishek Nadgeri,\textsuperscript{\rm 2}
    Kuldeep Singh,\textsuperscript{\rm 3}
       Toyotaro Suzumura,\textsuperscript{\rm 4} 
        Manish Singh\textsuperscript{\rm 1}
}
\affiliations {
    \textsuperscript{\rm 1} IIT Hyderabad,
    \textsuperscript{\rm 2} RWTH Aachen,
    \textsuperscript{\rm 3} Zerotha Research and Cerence Gmbh,
    \textsuperscript{\rm 4} The University of Tokyo\\
    cs20resch11002@iith.ac.in, abhishek.nadgeri@rwth-aachen.de, kuldeep.singh1@cerence.com, suzumura@acm.org, msingh@cse.iith.ac.in
}

\usepackage{bibentry}

\begin{document}

\maketitle

\begin{abstract}
Learning on evolving(dynamic) graphs has caught the attention of researchers as static methods exhibit limited performance in this setting. The existing methods for dynamic graphs learn spatial features by local neighborhood aggregation, which essentially only captures the low pass signals and local interactions. In this work, we go beyond current approaches to incorporate global features for effectively learning representations of a dynamically evolving graph. 
We propose to do so by capturing the spectrum of the dynamic graph. Since static methods to learn the graph spectrum would not consider the history of the evolution of the spectrum as the graph evolves with time, we propose a novel approach to learn the graph wavelets to capture this evolving spectra.
Further, we propose a framework that integrates the dynamically captured spectra in the form of these learnable wavelets into spatial features for incorporating local and global interactions. Experiments on eight standard datasets show that our method significantly outperforms related methods on various tasks for dynamic graphs.
\end{abstract}

\section{Introduction}
\noindent Recently there has been tremendous progress in the domain of Graph Representation Learning \cite{khoshraftar2022survey}. The aim here is to develop novel methods to learn the features of graphs in a vector space. Such approaches have successful applications in the domain of image recognition \cite{han2022vision}, computational chemistry \cite{DBLP:conf/nips/YingCLZKHSL21}, drug discovery \cite{DBLP:journals/corr/abs-2202-05146} and Natural Language processing \cite{10.1145/3404835.3462809}. Although effective, the underlying graphs are static in nature. 

In many real-world scenarios, graphs are dynamic, for example, social networks, citation graphs, bank transactions, etc. For such cases, various approaches have been developed (see survey by \citet{kazemi2020representation}). Broadly, these methods aim to learn the evolving nature of graphs through spatial features relying on local neighborhood aggregation (local dependencies) \cite{evolvegcn,DBLP:journals/kbs/GoyalCC20}. For example, researchers \cite{evolvegcn,shi2021gaen} have resorted to using GNNs along with RNNs to capture the dynamic evolving nature of graphs. 
With initial successes, these methods are inherently limited to capturing local interactions, missing out on other important information. 
For example, in the case of a dynamic money transaction graph, having a skewed representation of fraud users among genuine users, there would exist links between these fraud and genuine users, thus, giving rise to high-frequency signals. In the case of local neighborhood aggregation (attending to low-frequency signals), the majority node(genuine) pattern will cause the fraudulent pattern to diminish. Thus, capturing global properties becomes necessary. Here, the global information would help identify the fraud pattern eventually assisting in identifying the criminal. 
Similarly, in the citation graph, the local properties will help to understand the category of paper. In contrast, global properties will help to understand the amount of interdisciplinary research across research domains. Hence, learning the global dependencies is crucial for dynamic graphs which is a relatively unexplored area in its scientific literature.

\begin{figure*}[ht!]
\centering
\includegraphics[width=0.88\linewidth]{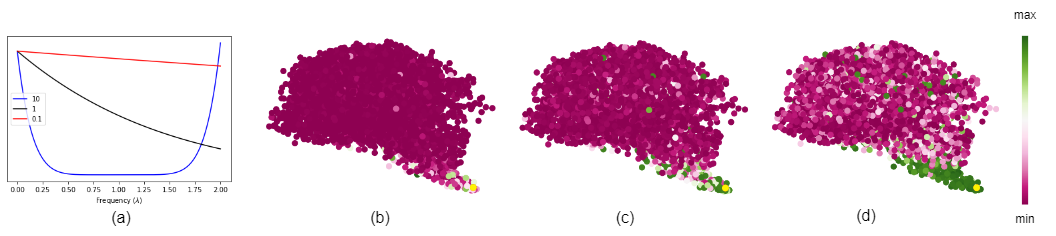}
\caption{For Brain dataset \cite{DBLP:conf/icdm/XuCLGLNZC019} at a given timestep, the figure (a) shows the learned filter functions at different scales by our proposed framework resembling band reject, low pass and all pass filters. The three diagrams to the right show the corresponding wavelets for a node (colored in yellow in figures (b,c,d) at lowest tip of the Brain) in the brain stem region. The graph nodes indicate the Regions of Interest(ROI) in brain. As scales change, the ROIs that the wavelet focuses on change from concentrated(local interactions) to diffused(global interactions). Moreover, the wavelets respect the brain structure and focus on the concerned region(brain stem in this case) thus mitigating noise due to interactions from unrelated regions.}
\label{fig:wavelets_motivation}
\vspace{-10pt}
\end{figure*}
In this paper we aim to encompass global dependencies in dynamic graphs, beyond local neighborhood aggregation, and view it through the lens of spectral graph theory \cite{hammond2019spectral}. 
For the same, we introduce a novel concept of \textit{learnable spectral graph wavelets} to capture global interactions of dynamic graph. The concept of learnable wavelet has following motivation: 
(i) Computing the spectra using full eigen decomposition is computationally expensive and requires $\mathcal{O}(N^3)$ time. Moreover, this gets even more computationally costly by a factor of the number of timesteps($T$) considered for dynamic graphs that evolve over every timestep. Spectral graph wavelets can be computed efficiently in $\mathcal{O}(N+|\mathcal{E}|)$ as we shall see in the following sections. (ii) Wavelets are sparser as compared to the graph Fourier transform adding to the computational benefits \cite{tremblay2014graph}. (iii) Wavelets give a sense of \textit{localization} in the \textit{vertex domain} of the graphs thus enabling interpretability of the convolutions while also being able to capture the global properties of the graph by changing the scale parameter (c.f., Figure \ref{fig:wavelets_motivation}). iv) Focusing on dynamic learnable wavelets helps in capturing global properties in that evolving spectra as the graph changes with time, where static and non-learnable wavelet methods \cite{xu2018graph} show empirical limitations. 

Furthermore, build on the recent success of capturing neighborhood features for evolving graphs, we propose to learn \textit{homogeneous representation} of spatial and spectral features to apprehend both local and global dependencies. 
Our approach is very intuitive as it keeps the proven local information intact whilst adding the global properties through learnable graph wavelet approach. Similar approaches have been proven in computer vision \cite{srinivas2021bottleneck} and NLP \cite{prasad2019glocal} where such restrictive inductive bias work well on learning local properties but miss out on global interactions. Also, the local aggregation leads to problems such as over smoothing and may not work well on heterophilic graphs \cite{wu2021representing}. Thus, using our methods will help to alleviate the drawbacks of the existing popular methods on dynamic graphs.

Our key contributions are two-fold: 1) a novel approach to \textit{learning spectral wavelets} on dynamic graphs for capturing global dependencies (with its theoretical foundations),
2) a novel framework named DEFT that combines spectral features obtained using learnable wavelets into spatial feature of the evolving graphs.
For effective use in downstream tasks, DEFT integrates the spatial and spectral features into \textit{homogeneous representations} which allows capturing shift invariance \cite{oppenheim1975digital} among the node features that could arise from the temporal nature.

\section{Related Works}
There has been considerable work on static graphs from the spatial and spectral domain perspective. Some works such as \cite{kipf2016gcn,hamilton2017inductive,velivckovic2018graph} focus on effectively learning spatial properties. Similarly, efforts such as \cite{levie2018cayleynets,balcilar2020analyzing,bastos2022how} have looked at graphs from a spectral perspective. GWNN \cite{xu2018graph}  has proposed to use spectral graph wavelets on static graphs. However, it obtains static wavelets using heat kernel-based filters.
Unlike GWNN, we learn spectral graph wavelets for dynamically evolving graphs. We observe in the experiment section that learnable wavelets perform significantly better for dynamic graphs compared to static wavelet methods.

One straight-forward way to use methods developed for static graphs on evolving dynamic graphs is to use RNN modules such as GRU and LSTM in addition to GNN modules to capture the evolving graph dynamism. This idea has been explored in works such as \cite{DBLP:journals/corr/SeoDVB16,narayan2018learning,DBLP:journals/pr/ManessiRM20}.
However, these models suffer a performance drop if new nodes are introduced in the graph; the GNNs may not be able to adapt to the evolving graph structure. Thus, EvolveGCN \cite{evolvegcn} has introduced to use RNNs to learn the parameters of the evolving GNN. The GNN, a GCN \cite{kipf2016gcn} in this case, is thus used ahead of the RNN module that captures the graph dynamism and offers promising results on dynamic graphs. EvolveGCN has limitation that it generalizes to unseen nodes in future timesteps, which is not always the case in real-world scenarios. \citet{xu2019gin} proposed similar approach by learning parameters of a GAT with an RNN that focus on graph topology discrepancies. 

Furthermore, autoencoder-based methods have been introduced, such as in \cite{DBLP:journals/corr/abs-1805-11273,DBLP:journals/kbs/GoyalCC20,xu2022dyng2g} that focus on reconstructing the graphs in future timesteps as an objective.
\citet{ledg} proposed a meta-learning framework in which the objective is to predict the graph at future timestep. However, these works focus only on learning spatial properties to capture local dependencies and ignores the global dependencies that may emerge due to dynamic nature of the graph.

\section{Preliminaries}
Consider a graph with vertices and edges as $(\mathcal{V},\mathcal{E})$ and adjacency matrix $A$. The laplacian($L=D-A$) can be decomposed into its orthogonal basis, namely the eigenvectors($U$) and eigenvalues($\Lambda$) as:$L = U \Lambda U^{*}$.
Let $X \in R^{N \times d}$ be the signal on the nodes of the graph. The Fourier Transform $\hat{X}$ of $X$ is then given as: $ \hat{X} = U^{*} X$. 

Spectral graph wavelet transforms \cite{hammond2011wavelets}  are obtained by functions of the laplacian $L$. 
Consider a linear self-adjoint operator($g(L)$) that acts on each component in the spectral space. We define a parameter $s$ for the scaling in the spectral space. The spectral graph wavelet at any given vertex $n$ is defined as the impulse($\delta_n$) response of the wavelet operator at vertex $n$: 
\[\psi_{s,n}(m) = \sum_{k=1}^{N} g(s\lambda_k) U_k^{*}(n) U_k(m) \]
The n-th wavelet coefficients at scale $s$ represented as $W_f(s,n)$ can be obtained by taking the inner product of the function $f$ in the \textit{vertex domain} with these wavelets as  
\begin{equation}\label{wavelet_coeff_prelim}
W_f(s,n) = \left< \psi_{s,n}, f \right> = \sum_{k=1}^{N} g(s\lambda_k) \hat{f}(k) U_k(m) 
\end{equation}

In our work, we propose to learn the wavelet coefficients for dynamic graphs, where the exact form of the scaling function $g(s \lambda_k)$ is parameterized.

\begin{figure*}[ht]
\centering
\centering
\subfigure[]{ 
\centering
  \includegraphics[width=0.60\linewidth]{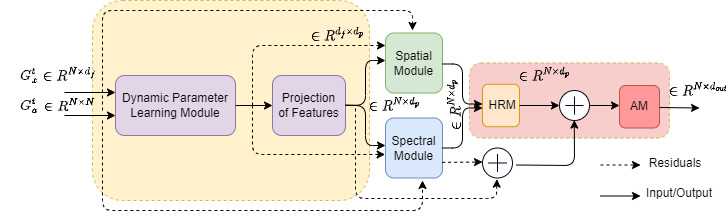}
}
\subfigure[]{ 
\centering
  \includegraphics[width=0.30\linewidth]{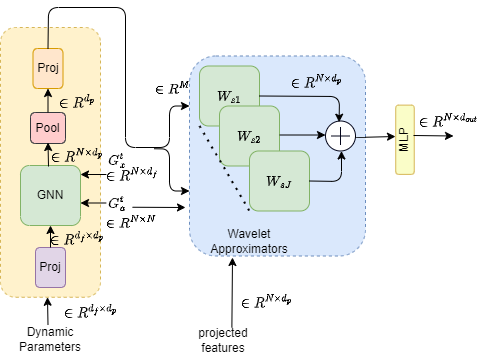}
}

\caption{Figure (a) shows the DEFT architecture. The yellow outer box shows the modules that learn the weight parameters of the GNNs in an evolving manner. These parameters are then given to the spectral and spatial modules to learn the corresponding features. The learned features then propagates to the \textit{homogeneous representation} module (HRM) followed by the aggregation module(AM) which together forms the integration module (red outer box) of our framework. Figure (b) explains the components of the spectral module. Similar to image (a), the yellow outer box learns the filter functions for the wavelet approximators.}
\label{fig:overall_architecture}
\vspace{-4pt}
\end{figure*}

\section{Method}
Figure \ref{fig:overall_architecture} illustrates the proposed DEFT (Dynamic wavElets For global inTeractions) framework that comprises the following modules: 1) Spectral component: focuses on global features of the graph in the form of learnable spectral graph wavelets. 
2) Spatial component: the necessity of this component is to mitigate the unsmooth spectrum phenomenon i.e. the node features gets correlated with eigenvectors corresponding to high frequencies causing information loss \cite{yang2022new}. To resolve this, spatial module focus on the local neighborhood of the graph in addition to that captured by the wavelets.
3) Integration module: Finally, for a \textit{homogeneous representation} of the global properties captured by the spectral component and the local properties learned by the spatial component, we propose a module that learns invariant representations and use these in an aggregation module for downstream tasks.

\subsection{Spectral Module}
We aim to capture global features without requiring the full eigen decomposition of the laplacian. Here, we propose to approximate the wavelet operator using some function.
We use the Chebyshev polynomials to be consistent with the literature \cite{hammond2011wavelets}. It is known that for any function $h$ with domain $y \in [-1,1]$ there exists a convergent Chebyshev series:
\begin{equation*}
    h(y) = \sum_{k=0}^{\infty} c_{k} T_k(x)
\end{equation*}
with the Chebyshev polynomials satisfying the recurrence relation $T_k(y)=2yT_{k-1}(y)-T_{k-2}(y), T_0=1, T_1=y$ and the coefficients $c_{k}$ are given by:
\begin{equation*}
    c_{k} = \frac{2}{\pi} \int_{-1}^{1} \frac{T_k(y)h(y)}{\sqrt{1-y^2}} dy = \frac{2}{\pi} \int_{0}^{\pi} cos(k\theta) h(cos(\theta)) d\theta
\end{equation*}
In order to approximate the function $g(s_j x)$ for each scale $j$, we need to bring the domain $x$ in $[-1,1]$. Noting that $0 \leq x \leq \lambda_{max}$ for the laplacian $L$, we perform the transform $y=\frac{x-a}{a}, a=\frac{\lambda_{max}}{2}$. We now define $\overline{T_k}(x) = T_k(y) = T_k(\frac{x-a}{a})$ and the approximation for $g$ looks as below
\begin{equation}\label{eq_chebyshev_approx}
    g(s_j x) = \sum_{k=0}^{\infty} c_{j,k} \overline{T_k}(x)
\end{equation}
with the coefficients given by,
\[c_{j,k} = \frac{2}{\pi} \int_{0}^{\pi} cos(k\theta) g(s_j (a(cos(\theta)+1))) d\theta\]
We truncate the polynomial to $M$ terms, which is the filter order. The coefficients $c_{j,k}$ which are analytical coefficients of the filter function as desired are approximated using functions parameterized by GNNs and MLPs, as we shall see next. $f_c^s$ is the parameterized form of it, obtained in spectral module at scale s. 
A GNN is used to perform message passing over the input graph along with the node features $v$ at layer $l$ for neighborhood $\mathcal{N}$. 
\begin{equation*}
    v_{im}^{l} = A_f({v_j^l | v_j \in \mathcal{N}(v_i)}), \ \ \ \ v_{i}^{l+1} = U_f(v_{i}^{l},v_{im}^{l})
\end{equation*}
Here, $A_f, U_f$ are the aggregation and update functions, respectively. The update function could contain a non-linearity such as leaky ReLU for better expressivity of the learned function. Since we intend to learn the filter coefficients $f_c \in R^M$ for the concerned graph($G$), we apply a pooling layer to get an intermediate vector representation($v_G \in R^{d_1}$) from the output of the GNN. The pooling layer converts a set of vectors(one for each node of the graph) to a single vector representation. For the final filter coefficients $f_c$, we apply a two-layer MLP with activation($\sigma$) to $v_g$
\begin{equation}
    f_c = W_2 \ \ \sigma \left( W_1 v_g \right) \\
\end{equation}
where $W_1 \in R^{d_2 \times d_1}, W_2 \in R^{d_2 \times M}$ are learnable weights. Since the two-layer MLP is a universal approximator \cite{hornik1991approximation} we can be assured of the existence of a function in this space that learns the desired mapping to the filter coefficients. In principle, any message passing GNN can be used to perform the update and aggregation steps. This process can be repeated with multiple GNNs for learning multiple filter functions.
As we consider dynamic graphs, we would like to evolve the parameters of the GNN with time (Dynamic Parameter learning Module of Figure \ref{fig:overall_architecture}, common for both spectral and spatial components). Inspired from \cite{evolvegcn}, we use an RNN module for generating the parameters for the GNN in layer $l$ at time $t$:
\begin{equation}
    W_t^l = RNN( H_{t}^{l}, W_{t-1}^{l} )
\end{equation}
where $W_t^l$ and $H_{t}^{l}$ are the hidden state and input at layer $l$ and time $t$ of the RNN.
In the below pseudo code, we outline our method to evolve the Spectral Module($ESpectral$) for dynamically learning filter coefficients per timestep 
\begin{algorithmic}[1]
  \small
  \baselineskip=15pt
  \Function {$f_{c_t} = ESpectral$}{$A_t, H_t^{(l)}, W_{t-1}^{(l)}$}
  \State $W_t^{(l)} = RNN(W_{t-1}^{(l)})$
  \State $H_t^{(l+1)} = GNN(A_t, H_t^{(l)}, W_t^{(l)})$
  \State $v_{g_t} = Pool(A_t, H_t^{(l+1)})$
  \State $f_{c_t} = W_2 \ \ \sigma \left( W_1 v_{g_t} \right)$
  \EndFunction
\end{algorithmic}
After learning GNN parameters, we need to learn filter coefficients for evolving graph. Learning the filter coefficients $f_c^s$ at a given timestep, we can obtain the wavelet operator $g(L)$ at scale $s=1$ using equation \ref{eq_chebyshev_approx}. For operators at a given scale $s > 0$, we could learn different parameterizations of the filter function at every scale. 
Note that approximating the functions at different scales in this manner would incur a storage cost of $\mathcal{O}(JN)$ for the filter coefficients. Along with this the storage and computation complexity would increase $J$ times for dynamically computing the filter coefficients from the GNNs. 
Thus we propose a ``rescaling trick'' wherein to obtain the operator at any scale $s > 0$, we perform the appropriate change of variables in equation \ref{eq_chebyshev_approx} to get $g(sL)$, keeping the coefficients $c_{j,k}$ fixed. That is for a scale $s$, instead of mapping $\lambda \xrightarrow[]{} g(\lambda)$ it would be mapped to $\lambda \xrightarrow[]{} g(s \lambda)$. 
It helps maintain parameter efficiency(by a factor of $J$) as the GNN weights(and also filter coefficients) are tied across all the scales.  Note here the exact filter learned would vary since the maximum frequency is the same. Hence, a bandpass at scale "one" may get converted to a highpass at scale "two".
Once we approximate $g(\lambda)$, we use it in learning the wavelet coefficients (output of spectral module) $W_f(s,n)$ as described in equation \ref{wavelet_coeff_prelim}.

Now, we give an approximation bound between the learned and desired filter function under the given framework for dynamic graphs with $N$ nodes and varying edges and signals. Please see appendix for all proofs.
\begin{lemma}\label{lemma_spectral1}
Consider $G^t(\lambda)$ to be the filter function at time $t$. Assume the Markov property to hold in the dynamic setting where the desired filter function($G^{t+1}(\Lambda)$) at time $t+1$ depends on the past state at time $t$($G^t(\Lambda)$). Consider this mapping between the past state and the current states to be captured by an arbitrary functional $f$ such that $G^{t+1}(\lambda) = f(G^t(\lambda_1), G^t(\lambda_2), \dots G^t(\lambda_N), \lambda)$ and we assume $f$ to be $L$ lipschitz continuous. Further, let $C_t=U_t G^t(\lambda) U_t^T \in R^{N \times N}$ represents the convolution support of the desired spectral filter and $C_t^a$ be the learnt convolution support at time $t$. Then, we have:

(i) $\norm{C_{t+1}^a - C_{t+1}}_F \leq LN^2 \sqrt{\norm{ C_{t}^a - C_{t} }_F^2 + \epsilon_{ca}^2} + \epsilon_{fa}$

(ii) $\norm{ C_{t+1}^a - C_{t}^a }_F \leq \norm{ C_{t+1} - C_{t} }_F + 2 \sqrt{N} \epsilon_{ca}$

\noindent where $\epsilon_{ca}$ and $\epsilon_{fa}$ are the filter polynomial(Chebyshev) and function approximation errors and depends on the order of the polynomial, number of training samples, model size etc.
\end{lemma}

Above result gives us a relation between the error at times $t$ and $t+1$ and has a factor of $N^2L$. Thus it requires the filter function to be smooth($L < \frac{1}{N^2}$) for convergence because under the given Markov assumptions, the past errors could accumulate in future timesteps. This is precisely why we need a gated model like GRU/LSTM in the Dynamic Parameter Generation module that can decouple the approximation of the filter function at a given timestep from the error in previous timesteps.
\subsection{Spatial Module}
\citet{yang2022new} concluded that high-frequency components of the signal on the graph get strengthened with a high correlation with each other, and the smooth signals become weak (unsmooth spectrum) i.e., the cosine similarity between the transformed signal and the low eigenvector reduces with the layers. In our setting, we illustrate that the factor by which the signal corresponding to the low-frequency component gets weakened is directly proportional to the magnitude of the frequency response at that frequency. 
\begin{lemma}
Let $G(\lambda)$ be the frequency response at frequency $\lambda$. Let $\lambda_1 \geq \lambda_2 \geq \dots \geq \lambda_n$ be the eigenvalues in descending order and $p_1, p_2, \dots p_n$ be the corresponding eigenvectors of the laplacian of the graph. Define $\lambda_{max}$ to be the eigenvalue at which $G(\lambda)$ is maximum. Let $C^{l}$ represent the convolution support($UG(\Lambda)U^T$) of the spectral filter at layer $l$. Then the factor by which the cosine similarity between consecutive layers dampens is $\underset{l \xrightarrow[]{} \infty}{lim} \frac{| cos(\left<C^{l+1}h, p_n\right>) |}{| cos(\left<C^{l}h, p_n\right> |} = \frac{G(\lambda_{n})}{G(\lambda_{max})}$.
\end{lemma}
If $\lambda_{max}$ belongs to one of the higher regions of the spectrum and $G(\lambda_n)$ is lower, then as the layers increase, the signals will lose the low-frequency information. 
Since the Spectral Module aims to capture the high-frequency components that may lead to an unsmooth spectrum, to resolve this, we explicitly strengthen the low-frequency components by using local neighborhood aggregation.
For this, we inherit the message passing GNNs, the parameters of which are generated using the RNN module inspired from \cite{evolvegcn}. The below pseudo code outlines the process to evolve the spatial module($ESpatial$).
\begin{algorithmic}[1]
  \small
  \baselineskip=15pt
  \Function {$[H_t^{(l+1)}, W_t^{(l)}] = ESpatial$}{$A_t, H_t^{(l)}, W_{t-1}^{(l)}$}
  \State $W_t^{(l)} = RNN(W_{t-1}^{(l)})$
  \State $H_t^{(l+1)} = GNN(A_t, H_t^{(l)}, W_t^{(l)})$
  \EndFunction
\end{algorithmic}

\subsection{Integration Module}
\textbf{Homogeneous representation Module(HRM)}: aims to achieve homogeneous representations from the spatial and spectral properties along with time features, which is essential for its usage in downstream tasks. 
A straightforward way is by concatenating two representations. However, due to the dynamic nature of the graph that evolves with time, we propose a learnable module that provides a notion of distance between the representations and helps induce position/structure information. Due to graph dynamism, it is important for the features to satisfy the shift invariance property:
\begin{definition}[Shift Invariance \cite{li2021learnable}]
Any two vectors $v_i = f_1(v_i^{'})$ and $v_j = f_1(v_j^{'})$ satisfy the shift invariance property if the inner product $\left< v_i, v_j \right>$ is a function of $v_i^{'}-v_j^{'}$ i.e. $\left< v_i, v_j \right> = f_2(v_i^{'}-v_j^{'})$.
\end{definition}
\noindent $f_1$ is an arbitrary function and $f_2$ is a linear transformation of $f_1$. Above property ensures that the relative distance between two nodes is maintained in the inner product space even if their absolute positions change(for example with addition of new nodes in the graph with time). Inspired from Fourier features  \cite{rahimi2007random,rahimi2008weighted}, for a node $i$, the spectral($v_{gi} \in R^{d_g}$) and spatial($v_{li} \in R^{d_l}$) embeddings are concatenated($\|$) along with the timestamp($t \in R^{d_t}$) information if available. Then, it is passed to an MLP.
\[v_{gli} = MLP(v_{gi} \| v_{li} \| t)\]
In order to obtain Fourier features($v_{ffi}$) from above intermediate representation, we take the element-wise sine,cosine and concatenate the two as:
\begin{align*}
    v_{ffi} = sin(v_{gli}) \| cos(v_{gli})
\end{align*}
\begin{property}\label{si_property}
The vector $v$ obtained by concatenation of the element wise sine and cosine of another vector $v^{'}$ i.e. $v = (sin(v^{'}) \| cos(v^{'}))$, satisfies the shift invariance property.
\end{property}
The property \ref{si_property} can be readily noted by observing that $cos(a-b) = cos(a)cos(b)+sin(a)sin(b)$. Taking the inner product($\left<.\right>$) of the above features for two nodes($i,j$) gives 
\begin{equation*}
    \small
    \left< v_{ffi}, v_{ffj} \right> = \sum cos(W_r(v_{gli} - v_{glj})) = ff_{W_r}(v_{gli} - v_{glj})
\end{equation*}
Property 1 is beneficial if these representations are used in an attention-based model such as \cite{vaswani_2017_attention,velivckovic2018graph} as we get a notion of closeness(similarity) in the embedding space.
The final \textit{homogeneous representation} for node $i$ ($v_{hri}$) is obtained as:
\begin{align*}
    v_{hri} = W_{hr2} \ \ \sigma (W_{hr1}(sin(v_{gli}) \| cos(v_{gli})))
\end{align*}
where $W_{hr2} \in R^{d \times d}, W_{hr1} \in R^{d \times (d_g+d_l+d_t)}$ are learnable weights and $\sigma$ is the activation function. 

\noindent \textbf{Aggregation(AM)}: Once we achieve homogeneous representations for the features, we can use these in the downstream task by applying a layer of MLP.
However, we also perform another level of aggregation to learn effective representations. While, in principle, we could use any of the existing message passing frameworks for this aggregation, we adopt a sparse variant of the attention mechanism inspired by \cite{vaswani_2017_attention} for computational benefits. 
Specifically, consider $X \in R^{N \times d_f}$ to be the node feature learned from the spatial and spectral modules. Now we define for the $l$-th layer, $W^{l}_{Q}, W^{l}_{K} \in R^{d_{out} \times d_f}$ to be the learnable weight matrices for the query and key of the self attention respectively. We apply self attention on the transformed features followed by softmax to get the aggregation weights $w_{ij}^l = softmax(\sum_{d_k} \hat{w_{ij}^{l}})$ between nodes $i,j$, where $ \hat{w_{ij}^{l}} = \frac{W^{l}_{Q} X[i]^T \odot W^{l}_{K} X[j]^T}{d_{out}}$ if nodes $i$ and $j$ are connected in the graph and 0 otherwise.
\subsection{Overall complexity}
The complexity of spectral module is $\mathcal{O}(|E| + N\sum_{j=0}^{J} M_j)$(cf., appendix) where $M_j$ is the order of the polynomial of the $j$-th filter head. The spatial modules can compute the features in a $\mathcal{O}(|E| + N)$ complexity. The integration module further has two components: HRM and AM. The HRM module would have a computational complexity of $\mathcal{O}(N)$ whereas for the AM, it depends on the underlying aggregator. In our choices, it would be an $\mathcal{O}(|E| + N)$ complexity. Thus the overall computational complexity comes to $\mathcal{O}(3|E| + N(3+\sum_{j=0}^{J} M_j) ) = \mathcal{O}(|E| + N)$, for bounded degree graphs further reduces to $\mathcal{O}(N)$ for a given snapshot at time $t$.

\section{Experiments} \label{sec:experiments}
Now, we present comprehensive experiments to evaluate our proposed framework. 
We borrow datasets and its preprocessing/splitting settings used in previous best baselines \cite{ledg,evolvegcn} with an average of five runs for final values. Code will be released.

\noindent \textbf{Datasets}:
Table \ref{tab:dataset} summarizes eight datasets for link prediction, edge classification, and node classification.
Each dataset contains a sequence of time-ordered graphs. SBM is a synthetic dataset to simulate evolving community structures. BC-OTC is a who-trusts-whom transaction network where the node represents users, and the edges represent the ratings that a member gives others in a range of -10(maximum distrust) to +10(maximum trust). BC-Alpha is similar to BC-OTC, albeit the transactions are on a different bitcoin network. UCI dataset is a student community network where nodes represent the students, and the edges represent the messages exchanged between them. AS dataset summarizes a temporal communication network indicating traffic flow between routers. In the Reddit dataset, nodes are the source and target subreddits, and the edges represent the sentiment's polarity between users.
The Elliptic dataset represents licit vs illicit transactions on the elliptic network of bitcoin transactions. The nodes represent transactions and edges represent payment flows. 
Finally, for Brain dataset, nodes represent tiny regions/cubes in the brain, and the edges are their connectivity. 

\begin{table}[ht] 
\small
\vspace{-24pt}
\centering
\caption{Data set statistics and details. In the task column, LP refers to Link Prediction, EC refers to Edge Classification and NC is Node Classification.}
\label{tab:dataset}
\vspace{-5pt}
\begin{tabular}{ccccc}
\hline
& \# Nodes & \# Edges & \# Time Steps & Task\\
&          &          & (Train / Val / Test) &  \\
\hline

BC-OTC   & 5,881   & 35,588  & 95 / 14 / 28 & EC\\
BC-Alpha & 3,777   & 24,173  & 95 / 13 / 28 & EC\\
Reddit   & 55,863  & 858,490 & 122 / 18 / 34 & EC\\
SBM      & 1,000   & 4,870,863 & 35 / 5 / 10 & LP \\
UCI      & 1,899   & 59,835  & 62 / 9 / 17 & LP\\
AS       & 6,474   & 13,895  & 70 / 10 / 20 & LP\\ 
Elliptic & 203,769 & 234,355 & 31 / 5 / 13 & NC\\
Brain   & 5,000  & 1,955,488 & 10 / 1 / 1 & NC\\
\hline
\vspace{-20pt}
\end{tabular}
\end{table}

\noindent \textbf{Baselines}:
In order to show the efficacy of our framework, we compare with the following competitive baselines: (1) \textit{Static graph methods}: GCN \cite{kipf2016gcn}, GAT \cite{velivckovic2018graph}. (2) \textit{Dynamic graph methods} such as GCN-GRU \cite{evolvegcn}, DynGEM \cite{DBLP:journals/corr/abs-1805-11273}, GAEN \cite{shi2021gaen}, the variants of  dyngraph2vec \cite{DBLP:journals/kbs/GoyalCC20}: dyngraph2vecAE(v1) and dyngraph2vecAERNN (v2), the best performing two variants of EvolveGCN \cite{evolvegcn}: EvolveGCN-O and EvolveGCN-H, and the two variants of LEDG \cite{ledg}: LEDGCN and LEDGAT.

\noindent \textbf{Task and Experiment Settings}:
For the link prediction task, we report mean reciprocal rank(MRR) and mean average precision(MAP) as in the baselines. For MRR, the existing links shall be ranked higher than non existing links. For each of the given nodes $u,v$, our model gives embeddings in the final layer $h_u^t, h_v^t$ respectively. We predict the link between these nodes by concatenating the vectors for the corresponding nodes. Following the standard practices, we perform negative sampling and optimize the cross entropy loss function. 
For edge classification, our model classifies the edges between two nodes($u,v$) by concatenating the vectors for the corresponding nodes($h_u^t, h_v^t$) at time $t$ and we report the micro F1 score. For node classification similar to the baselines, our model reports micro F1 (Brain dataset) or minority class F1 (Elliptic dataset) by classifying given node nodes($u$) as belonging to a certain class at time $t$. 

We perform a search for the best model from a set of  hyperparameters the range of which are as follows: Number of layers $\in \{1,2\}$, Hidden dimension $\in \{32,64,128\}$, Number of heads $\in \{4,8,16\}$, Filter order $\in \{4,8,16\}$, 
Wavelet scales $\in \left[ 0.1,10 \right]$. 
The rest of the parameters and settings are borrowed from previous works \cite{evolvegcn,ledg}.  All our experiments run on a single Tesla P100 GPU. The code will be public post-acceptance.

\noindent \textbf{Variants of our Framework}: In the main results we show three variants of our method with different aggregators. We use the MLP(DEFT-MLP), GAT(DEFT-GAT) and sparse Transformer(DEFT-T) as aggregators. The rationale behind selecting GAT and MLP in addition to the proposed transformer aggregator is to contrast it with a less expressive static attention mechanism in GAT and a control in MLP. 

\begin{table}[ht] \small
    \vspace{-24pt}
	\centering
	\caption{Link prediction results where mean average precision (MAP) and mean reciprocal rank (MRR) are displayed. Best values are bold, second bests are underlined.}\vspace{-3pt}
	\scalebox{0.8}{
	\begin{tabular}{c|cc|cc|cc}
	\hline
	\multirow{2}{*}{}Datasets&
	\multicolumn{2}{c|}{SBM}&
	\multicolumn{2}{c|}{UCI}&
	\multicolumn{2}{c}{AS}\cr\cline{2-7}
	Metrics&MAP&MRR&MAP&MRR&MAP&MRR
	\cr
	\hline
	GCN  & 0.1894 & 0.0136 & 0.0001 &  0.0468 & 0.0019 & 0.1814 \cr
	GAT  & 0.1751 & 0.0128 & 0.0001 &  0.0468 & 0.0200 & 0.1390 \cr
	DynGEM & 0.1680 & 0.0139 & 0.0209 &  0.1055 & 0.0529 & 0.1028 \cr
	GCN-GRU  & 0.1898 & 0.0119 & 0.0114& 0.0985 & 0.0713 & 0.3388\cr
	dyngraph2vec \tiny{V1} & 0.0983 & 0.0079 & 0.0044 &  0.0540 & 0.0331 & 0.0698 \cr
	dyngraph2vec \tiny{V2} & 0.1593 & 0.0120 & 0.0205 &  0.0713 & 0.0711 & 0.0493  \cr
    GAEN & 0.1828	& 0.0078 & 0.0001 &0.0490 & 0.1304 & 0.0507 \cr
	EvolveGCN-H & 0.1947 & 0.0141 & 0.0126 &  0.0899 & 0.1534 & 0.3632\cr
	EvolveGCN-O  & \underline{0.1989} & 0.0138 & 0.0270 & 0.1379 & 0.1139 & 0.2746\cr
	LEDG-GCN & 0.1960 & \underline{0.0147} & 0.0324   & 0.1626 & 0.1932 & \underline{0.4694}\cr
	LEDG-GAT &0.1822  & 0.0123     & 0.0261 & 0.1492  & \underline{0.2329} & 0.3835       \cr
	\hline
	\textbf{DEFT-MLP\tiny{(ours)}} &0.1658  & 0.0124     & \textbf{0.0543} & \underline{0.1715}  & 0.1737	&0.3569      \cr 
	\textbf{DEFT-GAT\tiny{(ours)}} &0.0966  & 0.0081     & \underline{0.0502} & 0.1702  & 0.0308 & 0.0928      \cr
	\textbf{DEFT-T\tiny{(ours)}} &\textbf{0.2421} & \textbf{0.0220} & 0.0501   & \textbf{0.2007} & \textbf{0.5879} & \textbf{0.6471}\cr
	\hline
	\end{tabular}}
	\vspace{-8pt}
	\label{tab:performance_lp}
\end{table}

\begin{figure}[ht]
\centering
\includegraphics[width=0.89\linewidth]{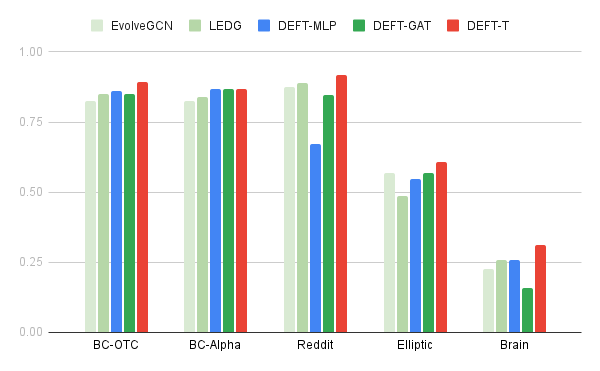}
\caption{Performance of edge classification and node classification, with F1 scores on the y-axis. For BC-OTC, BC-Alpha, and Reddit (edge classification) and Brain(node classification) the F1 score is the micro average. For Elliptic dataset(node classification), the F1 is the minority F1 because of the class imbalance and interest in the minority class (illicit transactions).}
\label{fig:edge_classification}
\vspace{-20pt}
\end{figure}

\section{Results}

\subsection{Results on Link Prediction}
Table \ref{tab:performance_lp} summarizes link prediction results. The key results are: 1) 
across datasets, our model (DEFT-T) significantly outperform all baselines which focus on learning local dependencies (in static and dynamic graph settings). It illustrates our framework's effectiveness in learning homogeneous representations of local and global dependencies in evolving graphs. 2) Interestingly, a simple MLP aggregator (DEFT-MLP) on the learned spectral and spatial properties already achieves better results than most of the baselines on UCI and AS datasets (including ones for dynamic graphs). 3) The transformer variant (DEFT-T) gives consistently better results than the GAT and MLP variants indicating the aggregation using the sparse self-attention of transformers is essential for the framework on these datasets. 

\subsection{Results on Edge Classification}
The results of edge classification on the three datasets (BC-OTC, BC-Alpha, Reddit) are given in Figure \ref{fig:edge_classification} and show the micro averaged F1 score. We observe from the results that at least one variant of our method outperforms the best baselines variants: EvolveGCN and LEDG. Our MLP-based configuration achieves better results than baselines on BC-OTC and BC-Alpha, whereas DEFT-T consistently outperforms the baselines and other DEFT variants. This further verifies the efficacy of the model developed under the proposed framework.

\subsection{Results on Node Classification}
Figure \ref{fig:edge_classification} shows the results of node classification on Brain(homophily ratio \cite{DBLP:conf/nips/ZhuYZHAK20}: 26\%) and Elliptic(homophily ratio: 96\%) datasets. 
DEFT would consider the brain structure on the Brain dataset while capturing global interactions. DEFT has a substantial performance gain on the Brain dataset, which has a high heterophily ratio requiring capturing global interactions in graphs compared to existing models that only perform local aggregation. 
On the homophilic Elliptic dataset, DEFT also performs good by capturing low-frequency signals. It confirms that when the spatial and spectral properties are combined using an optimal aggregator to learn the local and global dependencies, the model results in robust performance across tasks.

\subsection{Ablation Studies}

\begin{table}[ht] 
	\centering
	\caption{Table shows the link prediction results for the static wavelet and transformer architectures along with ablations of individual components of our model.}\vspace{-3pt}
	\scalebox{0.8}{
	\begin{tabular}{c|cc|cc}
	\hline
	\multirow{2}{*}{}Datasets&
	\multicolumn{2}{c|}{SBM}&
	\multicolumn{2}{c}{UCI}\cr\cline{2-5}
	Metrics&MAP&MRR&MAP&MRR
	\cr
	\hline
	GWNN  & 0.1789 & 0.0121 & 0.0076 & 0.0820\cr
	Transformer & 0.2052 & 0.0174 & 0.0308 &  0.1441  \cr
	GT & 0.2166 & 0.01805 & 0.0310 &  0.1414  \cr
	SAN & 0.2143 & 0.0180 & 0.0388 &  0.1822   \cr
	\hline
	DEFT-woSpectral  & 0.2123 & 0.0171 & 0.0460 &  0.1764 \cr
	DEFT-woSpatial  & 0.2083 & 0.0171 & 0.0421& 0.1806\cr
	DEFT-woHRM  & 0.2410 & 0.0217 & 0.0429& 0.1588\cr
	DEFT-staticSpectral & 0.2297 & 0.0198 & 0.0419 &  0.1506 \cr
	\hline
	\textbf{DEFT-T\tiny{(best configuration)}} &0.2421 & 0.0220 & 0.0501   & 0.2007 \cr
	\hline
	\end{tabular}}
	\label{tab:performance_ablation}
\end{table}

\textbf{Model Ablation Study:}
To study the effect of each module of DEFT, we systematically remove modules and create several sub-configuration of our best performing model (DEFT-T): (i) \textbf{DEFT-woSpectral} does not have the spectral module (ii) \textbf{DEFT-woSpatial} does not have the spatial module (iii) \textbf{DEFT-woHRM} with the Homogeneous Representation Module removed(the spatial and spectral features are simply added), 
(iv) \textbf{DEFT-staticSpectral} has the spectral wavelets that are learnable
but do not evolve with time and in a graph-specific manner. We also compare our best model with recent transformer-based models used for static graphs such as transformer \cite{vaswani_2017_attention}, SAN \cite{san2021}, GT \cite{dwivedi2020generalization}, and static wavelet baseline GWNN \cite{xu2018graph}.

The ablation study (Table \ref{tab:performance_ablation}) provides several key insights to understand proposed DEFT framework. Firstly, the lesser values reported by static (non-learnable) wavelet GNN (i.e., GWNN) and the static (but learnable) variant of our model (i.e., DEFT-staticSpectral) conclude that learning dynamic wavelets benefits in the case of evolving graphs. Secondly, the significant drop in performance when the spatial or spectral module is removed provides an essential finding that for dynamic graphs, these modules learn useful and orthogonal representations. When these modules are combined using an effectively learned homogeneous representation, the model (DEFT-T) provides significantly superior performance than its variants. Lastly, the static transformer models consistently report lesser values than our model.

\begin{figure}[t]
	\centering
	\includegraphics[width=0.35\textwidth]{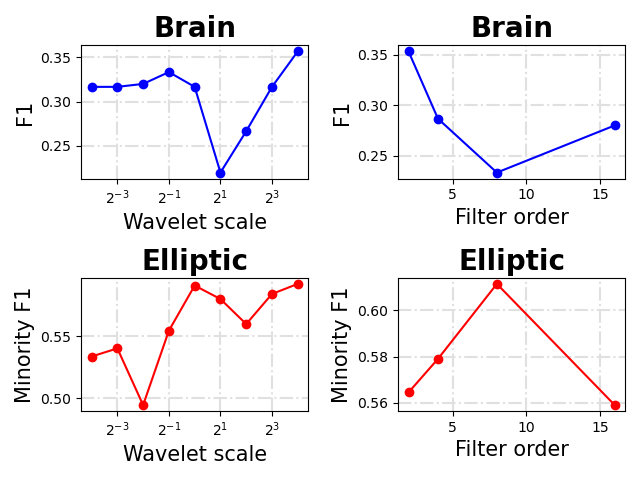} \vspace{-3.5pt}
    \caption{Effect of the parameters wavelet scale and filter order on the performance of the model.}
	\label{parameter_ws_fo}
\end{figure}

\textbf{Filter parameter selection:}
In spectral module, there are two parameters to be tuned: filter(polynomial) order and the wavelet scale. In this study, we try to understand how selection of parameters affects model performance.
From Figure \ref{parameter_ws_fo} illustrating results on Brain and Elliptic datasets, we make the following observations: 
(1) \textit{Wavelet scales}: In Brain dataset, the result dips with higher scale and then increases. For Elliptic, a general increasing trend is observed. 
In the case of the learnt filters for the given graphs, it seems that as the scales decrease the aggregation range of the spectral GNN increases(as in figure \ref{fig:wavelets_motivation}). 
Thus, brain dataset being heterophilic seems to benefit from scales in the lower region. Larger scales may work well for the homophilic Elliptic dataset as the aggregation focuses on a relatively smaller neighborhood, while also capturing the heterophily signals in the minority illicit transactions.
(2) \textit{Filter order}: On Brain dataset, there is a decreasing trend with increasing filter order, whereas, for Elliptic, an optimal result is attained in the middle region. 
This could be because the Brain dataset being small needs simpler filters(such as single mode band reject etc.); with higher filter order, it may have overfit. 
The larger Elliptic dataset, on the other hand, benefits from a larger filter order. Despite being homophilic for the majority nodes, it may need complex filters to handle the minority nodes due to the peculiar and global nature of illicit transactions. Thus, it benefits from relatively higher filter orders. 
These observations conclude that the parameters of the order of polynomial and the wavelet scales are dataset dependent.

\section{Conclusion}
Our work proposes a concept of \textit{learnable spectral graph wavelets} to capture global dependencies by omitting the need for full eigen decomposition of dynamic graphs. Furthermore, we implement it in the DEFT framework that integrates the graph's spectral and spatial properties for dynamic graphs. 
From the results on a wide range of tasks for dynamic graphs, we infer that the proposed method is able to capture the local(short range) and global(extended range) properties effectively. 
Also, it can capture a broad range of interactions respecting the graph structure, thus avoiding noise while being sparse(computationally efficient). Future works could build upon frameworks for joint time and graph Fourier transform for dynamic graphs to generate wavelets and study if these could be used for graph sampling to provide computational efficiency on extremely large graphs.

\section{Acknowledgments}
This work was partly supported by JSPS KAKENHI Grant Number JP21K21280. We thank anonymous reviewers for constructive feedback.

\bibliography{main}

\begin{thebibliography}{41}
\providecommand{\natexlab}[1]{#1}

\bibitem[{Balcilar et~al.(2020)Balcilar, Renton, H{\'e}roux, Ga{\"u}z{\`e}re,
  Adam, and Honeine}]{balcilar2020analyzing}
Balcilar, M.; Renton, G.; H{\'e}roux, P.; Ga{\"u}z{\`e}re, B.; Adam, S.; and
  Honeine, P. 2020.
\newblock Analyzing the expressive power of graph neural networks in a spectral
  perspective.
\newblock In \emph{International Conference on Learning Representations}.

\bibitem[{Barron(1991)}]{Barron1991ApproximationAE}
Barron, A. 1991.
\newblock Approximation and Estimation Bounds for Artificial Neural Networks.
\newblock \emph{Machine Learning}, 14: 115--133.

\bibitem[{Bastos et~al.(2022)Bastos, Nadgeri, Singh, Kanezashi, Suzumura, and
  Mulang'}]{bastos2022how}
Bastos, A.; Nadgeri, A.; Singh, K.; Kanezashi, H.; Suzumura, T.; and Mulang',
  I.~O. 2022.
\newblock How Expressive are Transformers in Spectral Domain for Graphs?
\newblock \emph{Transactions on Machine Learning Research}.

\bibitem[{Dwivedi and Bresson(2020)}]{dwivedi2020generalization}
Dwivedi, V.~P.; and Bresson, X. 2020.
\newblock A Generalization of Transformer Networks to Graphs.
\newblock arXiv:2012.09699.

\bibitem[{Goyal, Chhetri, and Canedo(2020)}]{DBLP:journals/kbs/GoyalCC20}
Goyal, P.; Chhetri, S.~R.; and Canedo, A. 2020.
\newblock dyngraph2vec: Capturing network dynamics using dynamic graph
  representation learning.
\newblock \emph{Knowl. Based Syst.}, 187.

\bibitem[{Goyal et~al.(2017)Goyal, Kamra, He, and
  Liu}]{DBLP:journals/corr/abs-1805-11273}
Goyal, P.; Kamra, N.; He, X.; and Liu, Y. 2017.
\newblock DynGEM: Deep Embedding Method for Dynamic Graphs.
\newblock \emph{IJCAI Workshop on Representation Learning for Graphs}.

\bibitem[{Hamilton, Ying, and Leskovec(2017)}]{hamilton2017inductive}
Hamilton, W.; Ying, Z.; and Leskovec, J. 2017.
\newblock Inductive representation learning on large graphs.
\newblock In \emph{Advances in neural information processing systems},
  1024--1034.

\bibitem[{Hammond, Vandergheynst, and Gribonval(2011)}]{hammond2011wavelets}
Hammond, D.~K.; Vandergheynst, P.; and Gribonval, R. 2011.
\newblock Wavelets on graphs via spectral graph theory.
\newblock \emph{Applied and Computational Harmonic Analysis}, 30(2): 129--150.

\bibitem[{Hammond, Vandergheynst, and Gribonval(2019)}]{hammond2019spectral}
Hammond, D.~K.; Vandergheynst, P.; and Gribonval, R. 2019.
\newblock The spectral graph wavelet transform: Fundamental theory and fast
  computation.
\newblock In \emph{Vertex-Frequency Analysis of Graph Signals}, 141--175.
  Springer.

\bibitem[{Han et~al.(2022)Han, Wang, Guo, Tang, and Wu}]{han2022vision}
Han, K.; Wang, Y.; Guo, J.; Tang, Y.; and Wu, E. 2022.
\newblock Vision GNN: An Image is Worth Graph of Nodes.
\newblock \emph{arXiv preprint arXiv:2206.00272}.

\bibitem[{Hornik(1991)}]{hornik1991approximation}
Hornik, K. 1991.
\newblock Approximation capabilities of multilayer feedforward networks.
\newblock \emph{Neural networks}, 4(2): 251--257.

\bibitem[{Kazemi et~al.(2020)Kazemi, Goel, Jain, Kobyzev, Sethi, Forsyth, and
  Poupart}]{kazemi2020representation}
Kazemi, S.~M.; Goel, R.; Jain, K.; Kobyzev, I.; Sethi, A.; Forsyth, P.; and
  Poupart, P. 2020.
\newblock Representation learning for dynamic graphs: A survey.
\newblock \emph{J. Mach. Learn. Res.}, 21(70): 1--73.

\bibitem[{Khoshraftar and An(2022)}]{khoshraftar2022survey}
Khoshraftar, S.; and An, A. 2022.
\newblock A Survey on Graph Representation Learning Methods.
\newblock \emph{arXiv preprint arXiv:2204.01855}.

\bibitem[{Kipf and Welling(2017)}]{kipf2016gcn}
Kipf, T.~N.; and Welling, M. 2017.
\newblock Semi-Supervised Classification with Graph Convolutional Networks.
\newblock In \emph{5th International Conference on Learning Representations,
  {ICLR} 2017}.

\bibitem[{Kreuzer et~al.(2021)Kreuzer, Beaini, Hamilton, L{\'{e}}tourneau, and
  Tossou}]{san2021}
Kreuzer, D.; Beaini, D.; Hamilton, W.~L.; L{\'{e}}tourneau, V.; and Tossou, P.
  2021.
\newblock Rethinking Graph Transformers with Spectral Attention.
\newblock \emph{NeurlPS 2021}, abs/2106.03893.

\bibitem[{Levie et~al.(2018)Levie, Monti, Bresson, and
  Bronstein}]{levie2018cayleynets}
Levie, R.; Monti, F.; Bresson, X.; and Bronstein, M.~M. 2018.
\newblock Cayleynets: Graph convolutional neural networks with complex rational
  spectral filters.
\newblock \emph{IEEE Transactions on Signal Processing}, 67(1): 97--109.

\bibitem[{Li et~al.(2021)Li, Si, Li, Hsieh, and Bengio}]{li2021learnable}
Li, Y.; Si, S.; Li, G.; Hsieh, C.-J.; and Bengio, S. 2021.
\newblock Learnable Fourier Features for Multi-dimensional Spatial Positional
  Encoding.
\newblock In Beygelzimer, A.; Dauphin, Y.; Liang, P.; and Vaughan, J.~W., eds.,
  \emph{Advances in Neural Information Processing Systems}.

\bibitem[{Manessi, Rozza, and Manzo(2020)}]{DBLP:journals/pr/ManessiRM20}
Manessi, F.; Rozza, A.; and Manzo, M. 2020.
\newblock Dynamic graph convolutional networks.
\newblock \emph{Pattern Recognit.}, 97.

\bibitem[{Narayan and Roe(2018)}]{narayan2018learning}
Narayan, A.; and Roe, P.~H. 2018.
\newblock Learning graph dynamics using deep neural networks.
\newblock \emph{IFAC-PapersOnLine}, 51(2): 433--438.

\bibitem[{Oppenheim and Schafer(1975)}]{oppenheim1975digital}
Oppenheim, A.~V.; and Schafer, R.~W. 1975.
\newblock \emph{Digital signal processing}.
\newblock Prentice-Hall international editions. Prentice-Hall.
\newblock ISBN 0132141078.

\bibitem[{Pareja et~al.(2020)Pareja, Domeniconi, Chen, Ma, Suzumura, Kanezashi,
  Kaler, Schardl, and Leiserson}]{evolvegcn}
Pareja, A.; Domeniconi, G.; Chen, J.; Ma, T.; Suzumura, T.; Kanezashi, H.;
  Kaler, T.; Schardl, T.~B.; and Leiserson, C.~E. 2020.
\newblock EvolveGCN: Evolving Graph Convolutional Networks for Dynamic Graphs.
\newblock In \emph{The Thirty-Fourth {AAAI} Conference on Artificial
  Intelligence, {AAAI} 2020}, 5363--5370. {AAAI} Press.

\bibitem[{Prasad and Kan(2019)}]{prasad2019glocal}
Prasad, A.; and Kan, M.-Y. 2019.
\newblock Glocal: Incorporating global information in local convolution for
  keyphrase extraction.
\newblock In \emph{Proceedings of the 2019 Conference of the North American
  Chapter of the Association for Computational Linguistics: Human Language
  Technologies, Volume 1 (Long and Short Papers)}, 1837--1846.

\bibitem[{Rahimi and Recht(2007)}]{rahimi2007random}
Rahimi, A.; and Recht, B. 2007.
\newblock Random Features for Large-Scale Kernel Machines.
\newblock In \emph{Advances in Neural Information Processing Systems 20,
  Proceedings of the Twenty-First Annual Conference on Neural Information
  Processing Systems}, 1177--1184. Curran Associates, Inc.

\bibitem[{Rahimi and Recht(2008)}]{rahimi2008weighted}
Rahimi, A.; and Recht, B. 2008.
\newblock Weighted Sums of Random Kitchen Sinks: Replacing minimization with
  randomization in learning.
\newblock In \emph{Advances in Neural Information Processing Systems 21,
  Proceedings of the Twenty-Second Annual Conference on Neural Information
  Processing Systems}, 1313--1320. Curran Associates, Inc.

\bibitem[{Seo et~al.(2016)Seo, Defferrard, Vandergheynst, and
  Bresson}]{DBLP:journals/corr/SeoDVB16}
Seo, Y.; Defferrard, M.; Vandergheynst, P.; and Bresson, X. 2016.
\newblock Structured Sequence Modeling with Graph Convolutional Recurrent
  Networks.
\newblock \emph{CoRR}, abs/1612.07659.

\bibitem[{Shi et~al.(2021)Shi, Huang, Zhu, Tang, Zhuang, and Liu}]{shi2021gaen}
Shi, M.; Huang, Y.; Zhu, X.; Tang, Y.; Zhuang, Y.; and Liu, J. 2021.
\newblock GAEN: Graph Attention Evolving Networks.
\newblock In \emph{IJCAI}, 1541--1547.

\bibitem[{Srinivas et~al.(2021)Srinivas, Lin, Parmar, Shlens, Abbeel, and
  Vaswani}]{srinivas2021bottleneck}
Srinivas, A.; Lin, T.-Y.; Parmar, N.; Shlens, J.; Abbeel, P.; and Vaswani, A.
  2021.
\newblock Bottleneck transformers for visual recognition.
\newblock In \emph{Proceedings of the IEEE/CVF conference on computer vision
  and pattern recognition}, 16519--16529.

\bibitem[{St{\"{a}}rk et~al.(2022)St{\"{a}}rk, Ganea, Pattanaik, Barzilay, and
  Jaakkola}]{DBLP:journals/corr/abs-2202-05146}
St{\"{a}}rk, H.; Ganea, O.; Pattanaik, L.; Barzilay, R.; and Jaakkola, T.~S.
  2022.
\newblock EquiBind: Geometric Deep Learning for Drug Binding Structure
  Prediction.
\newblock \emph{CoRR}, abs/2202.05146.

\bibitem[{Tremblay and Borgnat(2014)}]{tremblay2014graph}
Tremblay, N.; and Borgnat, P. 2014.
\newblock Graph wavelets for multiscale community mining.
\newblock \emph{IEEE Transactions on Signal Processing}, 62(20): 5227--5239.

\bibitem[{Vaswani et~al.(2017)Vaswani, Shazeer, Parmar, Uszkoreit, Jones,
  Gomez, Kaiser, and Polosukhin}]{vaswani_2017_attention}
Vaswani, A.; Shazeer, N.; Parmar, N.; Uszkoreit, J.; Jones, L.; Gomez, A.~N.;
  Kaiser, {\L}.; and Polosukhin, I. 2017.
\newblock Attention is all you need.
\newblock In \emph{Advances in neural information processing systems},
  5998--6008.

\bibitem[{Veli{\v{c}}kovi{\'c} et~al.(2018)Veli{\v{c}}kovi{\'c}, Cucurull,
  Casanova, Romero, Li{\`o}, and Bengio}]{velivckovic2018graph}
Veli{\v{c}}kovi{\'c}, P.; Cucurull, G.; Casanova, A.; Romero, A.; Li{\`o}, P.;
  and Bengio, Y. 2018.
\newblock Graph Attention Networks.
\newblock In \emph{International Conference on Learning Representations}.

\bibitem[{Wu et~al.(2021{\natexlab{a}})Wu, Chen, Ji, and
  Liu}]{10.1145/3404835.3462809}
Wu, L.; Chen, Y.; Ji, H.; and Liu, B. 2021{\natexlab{a}}.
\newblock Deep Learning on Graphs for Natural Language Processing.
\newblock In \emph{Proceedings of the 44th International ACM SIGIR Conference
  on Research and Development in Information Retrieval}, 2651–2653.
  Association for Computing Machinery.

\bibitem[{Wu et~al.(2021{\natexlab{b}})Wu, Jain, Wright, Mirhoseini, Gonzalez,
  and Stoica}]{wu2021representing}
Wu, Z.; Jain, P.; Wright, M.; Mirhoseini, A.; Gonzalez, J.~E.; and Stoica, I.
  2021{\natexlab{b}}.
\newblock Representing long-range context for graph neural networks with global
  attention.
\newblock \emph{Advances in Neural Information Processing Systems}, 34:
  13266--13279.

\bibitem[{Xiang, Huang, and Wang(2022)}]{ledg}
Xiang, X.; Huang, T.; and Wang, D. 2022.
\newblock Learning to Evolve on Dynamic Graphs (SA).
\newblock In \emph{Thirty-Sixth {AAAI} Conference on Artificial Intelligence,
  {AAAI} 2022}. {AAAI} Press.

\bibitem[{Xu et~al.(2019{\natexlab{a}})Xu, Shen, Cao, Qiu, and
  Cheng}]{xu2018graph}
Xu, B.; Shen, H.; Cao, Q.; Qiu, Y.; and Cheng, X. 2019{\natexlab{a}}.
\newblock Graph Wavelet Neural Network.
\newblock In \emph{International Conference on Learning Representations}.

\bibitem[{Xu et~al.(2019{\natexlab{b}})Xu, Cheng, Luo, Gu, Liu, Ni, Zong, Chen,
  and Zhang}]{DBLP:conf/icdm/XuCLGLNZC019}
Xu, D.; Cheng, W.; Luo, D.; Gu, Y.; Liu, X.; Ni, J.; Zong, B.; Chen, H.; and
  Zhang, X. 2019{\natexlab{b}}.
\newblock Adaptive Neural Network for Node Classification in Dynamic Networks.
\newblock In Wang, J.; Shim, K.; and Wu, X., eds., \emph{2019 {IEEE}
  International Conference on Data Mining, {ICDM} 2019, Beijing, China,
  November 8-11, 2019}, 1402--1407. {IEEE}.

\bibitem[{Xu et~al.(2019{\natexlab{c}})Xu, Hu, Leskovec, and
  Jegelka}]{xu2019gin}
Xu, K.; Hu, W.; Leskovec, J.; and Jegelka, S. 2019{\natexlab{c}}.
\newblock How Powerful are Graph Neural Networks?
\newblock arXiv:1810.00826.

\bibitem[{Xu, Singh, and Karniadakis(2022)}]{xu2022dyng2g}
Xu, M.; Singh, A.~V.; and Karniadakis, G.~E. 2022.
\newblock DynG2G: An Efficient Stochastic Graph Embedding Method for Temporal
  Graphs.
\newblock \emph{IEEE Transactions on Neural Networks and Learning Systems}.

\bibitem[{Yang et~al.(2022)Yang, Shen, Li, Qi, Zhang, and Yin}]{yang2022new}
Yang, M.; Shen, Y.; Li, R.; Qi, H.; Zhang, Q.; and Yin, B. 2022.
\newblock A New Perspective on the Effects of Spectrum in Graph Neural
  Networks.
\newblock In \emph{International Conference on Machine Learning}, 25261--25279.
  PMLR.

\bibitem[{Ying et~al.(2021)Ying, Cai, Luo, Zheng, Ke, He, Shen, and
  Liu}]{DBLP:conf/nips/YingCLZKHSL21}
Ying, C.; Cai, T.; Luo, S.; Zheng, S.; Ke, G.; He, D.; Shen, Y.; and Liu, T.
  2021.
\newblock Do Transformers Really Perform Badly for Graph Representation?
\newblock In \emph{Advances in Neural Information Processing Systems 34: Annual
  Conference on Neural Information Processing Systems 2021, NeurIPS 2021},
  28877--28888.

\bibitem[{Zhu et~al.(2020)Zhu, Yan, Zhao, Heimann, Akoglu, and
  Koutra}]{DBLP:conf/nips/ZhuYZHAK20}
Zhu, J.; Yan, Y.; Zhao, L.; Heimann, M.; Akoglu, L.; and Koutra, D. 2020.
\newblock Beyond Homophily in Graph Neural Networks: Current Limitations and
  Effective Designs.
\newblock In \emph{Advances in Neural Information Processing Systems 33: Annual
  Conference on Neural Information Processing Systems 2020, NeurIPS 2020}.

\end{thebibliography}
\section{Supplementary Material}
In this section we outline the proofs for the theoretical results stated in the main paper. For completion we give the statement along with the proofs.

\begin{lemma}\label{lemma_spectral1}
Consider $G^t(\lambda)$ to be the filter function at time $t$. Assume the Markov property to hold in the dynamic setting where the desired filter function($G^{t+1}(\Lambda)$) at time $t+1$ depends on the past state at time $t$($G^t(\Lambda)$). Consider this mapping between the past state and the current states to be captured by an arbitrary functional $f$ such that $G^{t+1}(\lambda) = f(G^t(\lambda_1), G^t(\lambda_2), \dots G^t(\lambda_N), \lambda)$ and we assume $f$ to be $L$ lipschitz continuous. Further, let $C_t=U_t G(\lambda)^t U_t^T \in R^{N \times N}$ represent the convolution support of the desired spectral filter and $C_t^a$ be the learnt convolution support at time $t$. Then, we have:

(i) $\norm{C_{t+1}^a - C_{t+1}}_F \leq LN^2 \sqrt{\norm{ C_{t}^a - C_{t} }_F^2 + \epsilon_{ca}^2} + \epsilon_{fa}$

(ii) $\norm{ C_{t+1}^a - C_{t}^a }_F \leq \norm{ C_{t+1} - C_{t} }_F + 2 \sqrt{N} \epsilon_{ca}$

where $\epsilon_{ca}$ and $\epsilon_{fa}$ are the filter polynomial(Chebyshev) and function approximation errors and depends on the order of the polynomial, number of training samples, model size etc.
\end{lemma}
\begin{proof}
(i)
The approximate and desired convolution support matrix $C_{t+1}^a$ and $C_{t+1}$ at time $t$ of the spectral filter has the decomposition as:
\begin{align*}
    C_{t+1}^a &= U_{t+1} G^a_{t+1}(\Lambda) U_{t+1}^T \\
    C_{t+1} &= U_{t+1} G_{t+1}(\Lambda) U_{t+1}^T
\end{align*}
where $U_{t+1} \in R^{N \times N}$ is the matrix whose columns contain the eigenvectors of the graph laplacian at time $t+1$.
Now the Frobenius norms of these two matrices is given by,
\begin{align*}
    &\norm{ C_{t+1}^a - C_{t+1} } \\
    &= \norm { U_{t+1} G^a_{t+1}(\Lambda) U_{t+1}^T -  U_{t+1} G_{t+1}(\Lambda) U_{t+1}^T} \\
    &= \norm { U_{t+1} (G^a_{t+1}(\Lambda) - G_{t+1}(\Lambda)) U_{t+1}^T} \\
\end{align*}
As the graph evolves at time $t$, let $f_t$ represent the function that maps the desired spectra at time $t$ to the one at time $t+1$ i.e. $G_{t+1}(\lambda_i) = f_{t}(\|_{i \in [N]} G^t(\lambda_i), \lambda_i), \forall i \in [N]$. Let $f_{t}^a$ represent a similar mapping from $G_{t}^{a}(\Lambda)$ to $G_{t+1}^{a}(\Lambda)$. Using the approximation result of multi-layer neural networks in \cite{Barron1991ApproximationAE}, we have $f_{t}^{a}(\|_{i \in [N]} G^t(\lambda_i), \lambda_i ) - f_{t}(\|_{i \in [N]} G^t(\lambda_i), \lambda_i ) = \mathcal{O}(\frac{C_f^2}{h} + \frac{hd}{N}log(N)) = \epsilon_1$ where $C_f$ is the first absolute moment of the distribution of the Fourier magnitudes of $f_{t}^{a}$, $h$ the number of hidden units in the network, $d$ is the input dimension and $N_s$ the number of training samples, $\|$ is the concatenation operator.
Thus we have 
\begin{equation}
    f_{t}^{a}( \|_{i \in [N]} G^t(\lambda_i), \lambda) = f_{t}(\|_{i \in [N]} G^t(\lambda_i), \lambda) + \epsilon_1
\end{equation}
where $\epsilon_1 = \mathcal{O}(\frac{C_f^2}{h} + \frac{hd}{N}log(N))$.

Now, from the above equations we upper bound the norm(Frobenius unless mentioned otherwise) as,
\small{
\begin{align*}
    &\norm{ C_{t+1}^a - C_{t+1} } = \norm { U_{t+1} (G^a_{t+1}(\Lambda) - G_{t+1}(\Lambda)) U_{t+1}^T} \\
    &= \norm { U_{t+1} (f^a_{t}(G^a_t(\Lambda), \Lambda) - f_{t}(G_t(\Lambda), \Lambda)) U_{t+1}^T} \\
    &= \norm { U_{t+1} (f^a_{t}(G^a_t(\Lambda), \Lambda) + \epsilon_1 I_N - f_{t}(G_t(\Lambda), \Lambda)) U_{t+1}^T} \\
    &= \norm { U_{t+1} (f^a_{t}(G^a_t(\Lambda), \Lambda) - f_{t}(G_t(\Lambda), \Lambda)) U_{t+1}^T +  U_{t+1} (\epsilon_1 I_N U_{t+1}^T)} \\
    &\leq \norm { U_{t+1} (f^a_{t}(G^a_t(\Lambda), \Lambda) - f_{t}(G_t(\Lambda), \Lambda)) U_{t+1}^T} +  \norm{ \epsilon_1 U_{t+1} I_N U_{t+1}^T} \\
    &= \norm { U_{t+1} (f^a_{t}(G^a_t(\Lambda), \Lambda) - f_{t}(G_t(\Lambda), \Lambda)) U_{t+1}^T} +  \norm{ \epsilon_1 U_{t+1} U_{t+1}^T} \\
    &= \norm { U_{t+1} (f^a_{t}(G^a_t(\Lambda), \Lambda) - f_{t}(G_t(\Lambda), \Lambda)) U_{t+1}^T} + \epsilon_1  \norm{ I_N} \\
    &= \norm { U_{t+1} (f^a_{t}(G^a_t(\Lambda), \Lambda) - f_{t}(G_t(\Lambda), \Lambda)) U_{t+1}^T} + \sqrt{N}\epsilon_1 \\
    &= \norm { U_{t+1} (f^a_{t}(G^a_t(\Lambda), \Lambda) - f_{t}(G_t(\Lambda), \Lambda)) U_{t+1}^T} + \epsilon_{fa} \\
\end{align*}
}
where $\epsilon_{fa} = \sqrt{N}\mathcal{O}(\frac{C_f^2}{h} + \frac{hd}{N_s}log(N_s))$ and $f^a_{t}, f_{t}$ are applied element wise over the domain($\Lambda$ being discrete).
Assuming, $G^a_{t}(\Lambda) \approx G_{t}(\Lambda)$(which holds true asymptotically) we perform a first order Taylor series expansion of the function $f_{t}(G^a_t(\Lambda), \Lambda)$ at $(G^a_t(\Lambda), \Lambda) = \|_{i \in [N]} G^t(\lambda_i), \lambda_i$ as below
\small{
\begin{align*}
    f_{t}(G^a_t(\Lambda), \Lambda) &\approx f_{t}(G_t(\Lambda), \Lambda) + \left< (G^a_t(\Lambda)-G_t(\Lambda),0), f^{'}_{t}(G_t(\Lambda), \Lambda) \right> \\
\end{align*}
}
where $f^{'}_{t}(G_t(\Lambda), \Lambda) \in R^{N}$ is the vector of partial derivatives along each basis. Thus in the limit that $G^a_t(\Lambda) \xrightarrow[]{} G_t(\Lambda)$,
\begin{align*}
    &\underset{G^a_t(\Lambda) \xrightarrow[]{} G_t(\Lambda)}{lim} f_{t}(G^a_t(\Lambda), \Lambda) - f_{t}(G_t(\Lambda), \Lambda) \\
    &= \left< (G^a_t(\Lambda)-G_t(\Lambda),0), f^{'}_{t}(G_t(\Lambda), \Lambda) \right> \\
    &\therefore \norm{ f_{t}(G^a_t(\Lambda), \Lambda) - f_{t}(G_t(\Lambda), \Lambda) } \\
    &= \norm{\left< (G^a_t(\Lambda)-G_t(\Lambda),0), f^{'}_{t}(G_t(\Lambda), \Lambda) \right>} \\
    &\leq \norm{\left< (G^a_t(\Lambda)-G_t(\Lambda),0), L \mathbf{1} \right>} \\
    &\leq L \norm{ \sum_{i \in [N]} G^a_t(\lambda_i)-G_t(\lambda_i) } \\
\end{align*}
Now since $G^a_t(\lambda_i)$ is obtained from a truncated polynomial(Chebyshev) approximation which is parameterized using a (universal) function approximator, we would have error terms involving the polynomial as well as function approximation. From \citep{bastos2022how} which proved such bounds for the static case we have,

\[G^a_t(\lambda_i)-G_t(\lambda_i)  = \mathcal{O}(\frac{KC_f^2}{h} + \frac{hK^2}{N_s}log(N_s) + K^{-m}) \]
Here it is assumed that the input dimension is $\mathcal{O}(K)$ which is the filter order and the desired filter response $G(\lambda)$ has $m+1$ continuous derivatives on the domain $[-1,1]$. We set $\mathcal{O}(\frac{KC_f^2}{h} + \frac{hK^2}{N_s}log(N_s) + K^{-m}) = \epsilon_{ca}$ for notational simplicity.
Now we have that, 
\begin{align*}
    &\norm{ \sum_{i \in [N]} G^a_t(\lambda_i)-G_t(\lambda_i) }^2 = \norm{G^a_t(\Lambda)-G_t(\Lambda)}^2 \\
    &+ 2 \sum_{i \in [N], i \neq j} (G^a_t(\lambda_i)-G_t(\lambda_i))(G^a_t(\lambda_j)-G_t(\lambda_j)) \\
    &= \norm{G^a_t(\Lambda)-G_t(\Lambda)}^2 + 2 N^2 \epsilon_{ca}^2 \\
    &\norm{ \sum_{i \in [N]} G^a_t(\lambda_i)-G_t(\lambda_i) } = \sqrt{ \norm{G^a_t(\Lambda)-G_t(\Lambda)}^2 + 2 N^2 \epsilon_{ca}^2}
\end{align*}
From the above equation we have,
\begin{align*}
    &\norm{f_{t}(G^a_t(\Lambda), \Lambda) - f_{t}(G_t(\Lambda), \Lambda)} \leq L \norm{ \sum_{i \in [N]} G^a_t(\lambda_i)-G_t(\lambda_i) } \\
    &= L \sqrt{ \norm{G^a_t(\Lambda)-G_t(\Lambda)}^2 + 2 N^2 \epsilon_{ca}^2} \\
    &=  L \sqrt{ \norm{ U_t^T U_t (G^a_t(\Lambda)-G_t(\Lambda)) U_t^T U_t }^2 + 2 N^2 \epsilon_{ca}^2} \\
    &=  L \sqrt{ \norm{ U_t^T (U_t G^a_t(\Lambda) U_t^T - U_t G_t(\Lambda) U_t^T) U_t }^2 + 2 N^2 \epsilon_{ca}^2} \\
    &\leq  L \sqrt{ \norm{ U_t^T }^2 \norm{(U_t G^a_t(\Lambda) U_t^T - U_t G_t(\Lambda) U_t^T)}^2 \norm{ U_t }^2 + 2 N^2 \epsilon_{ca}^2} \\
    &\leq  L \sqrt{ N^2 \norm{(U_t G^a_t(\Lambda) U_t^T - U_t G_t(\Lambda) U_t^T)}^2 + 2 N^2 \epsilon_{ca}^2} \\
    &=  L \sqrt{ N^2 \norm{C^a_t - C_t}^2 + 2 N^2 \epsilon_{ca}^2} \\
    &=  LN \sqrt{ \norm{C^a_t - C_t}^2 + 2 \epsilon_{ca}^2} \\
\end{align*}
Using this and the expression for $\norm{ C_{t+1}^a - C_{t+1} }$ we have,
\begin{align*}
    &\norm{ C_{t+1}^a - C_{t+1} } \leq \norm { U_{t+1} (f^a_{t}(G^a_t(\Lambda), \Lambda) - f_{t}(G_t(\Lambda), \Lambda)) U_{t+1}^T} + \epsilon \\
    &\leq \norm{U_{t+1}} \norm{(f^a_{t}(G^a_t(\Lambda), \Lambda) - f_{t}(G_t(\Lambda), \Lambda))} \norm{U_{t+1}^T} \\
    &\leq L N^2 \sqrt{ \norm{C^a_t - C_t}^2 + 2 \epsilon_{ca}^2} + \epsilon_{fa} 
\end{align*}
This completes the first part of the proof.

\noindent (ii) For the second part we have to upper bound $\norm{C^a_{t+1} - C^a_t}$. We have,
\begin{align*}
    &\norm{C^a_{t+1} - C^a_t} = \norm{U_{t+1} G^a_{t+1}(\Lambda) U_{t+1} - U_t G^a_t(\Lambda) U_t^T} \\
    &= \norm{U_{t+1} (G_{t+1}+\epsilon_{ca}I_N)(\Lambda) U_{t+1} - U_t (G_t(\Lambda)+\epsilon_{ca}I_N) U_t^T} \\
    &\leq \norm{U_{t+1} G_{t+1} U_{t+1}^T - U_t G_t(\Lambda) U_t^T} + 2\epsilon_{ca}\norm{I_N} \\
    &= \norm{C_{t+1}-C_t} + 2 \sqrt{N} \epsilon_{ca}
\end{align*}
where $\epsilon_{ca}$ in the second step is the same as for the part (i). This completes the proof.

\end{proof}
Part (i) in the above result gives us a relation between the error at times $t$ and $t+1$. In the expression $LN^2 \sqrt{\norm{ C_{t}^a - C_{t} }_F^2 + \epsilon_{ca}^2} + \epsilon_{fa}$ if the error of approximating $C_t$ if larger than the chebyshev polynomial approximation error i.e. $\norm{ C_{t}^a - C_{t} }_F >> \epsilon_{ca}$, then we have,
\begin{align*}
    \norm{ C_{t+1}^a - C_{t+1} }_F &\leq LN^2 \sqrt{\norm{ C_{t}^a - C_{t} }_F^2 + \epsilon_{ca}^2} + \epsilon_{fa} \\
    &\approx LN^2 \sqrt{\norm{ C_{t}^a - C_{t} }_F^2} + \epsilon_{fa} \\
    &= LN^2 \norm{ C_{t}^a - C_{t} }_F + \epsilon_{fa} \\
    &= (LN^2)^{t} \norm{ C_{0}^a - C_{0} }_F + \frac{\epsilon_{fa}}{1-LN^2}
\end{align*}
The last equation follows if $L < \frac{1}{N^2}$ in the limit of $t \xrightarrow[]{} \infty$. Thus convergence is guaranteed if the filter function is smooth. This is also one of the reasons why we initialise the filter coeffcients with a one vector(all pass filter). However this may be too strict a constraint in practice, and in the cases where the initial approximation error is large and/or the filter function is complex with high gradients the approximation error of the convolution support may diverge with time. This issue arises as we have taken the Markov assumption, for theoretical analysis, for determining the filter function at the next state i.e. the filter function depends on the current filter function. However in real world settings the assumption may not hold and empirically learning may happen from other signals(eg: the current graph input etc.) and the filter function may not depend on the past state entirely. This is where the LSTM(of the dynamic parameter module) may help if it's forget gate decides that the current input is the useful signal for learning the filter function and not the previous state, thus decoupling the learning of the filter at current time step from the errors of the previous ones and empirically attaining a bounded error with time.
Similarly, part (ii) states that under the given assumptions, the deviation between the approximated convolution supports at consecutive time steps is directly related to that between the desired supports at the respective time steps. 

\begin{lemma}
Let $G(\lambda)$ be the frequency response at frequency $\lambda$. Let $\lambda_1 \geq \lambda_2 \geq \dots \geq \lambda_n$ be the eigenvalues in descending order and $p_1, p_2, \dots p_n$ be he corresponding eigenvalues of the laplacian of the graph. Define $\lambda_{max}$ to be the eigenvalue at which $G(\lambda)$ is maximum. Let $C^{l}$ represent the convolution support($UG(\Lambda)U^T$) of the spectral filter at layer $l$. Then the factor by which the cosine similarity between consecutive layers dampens is $\underset{l \xrightarrow[]{} \infty}{lim} \frac{| cos(\left<C^{l+1}h, p_n \right>) |}{| cos(\left<C^{l}h, p_n \right> |} = \frac{G(\lambda_{n})}{G(\lambda_{max})}$.
\end{lemma}
\begin{proof}
We begin by noting that the expression $cos(\left<C^{l} h, p_n \right>) = \frac{\alpha_n G^{l}(\lambda_n)}{\sqrt{\sum_{i=1}^{N} \alpha_i^2 G^{2l}(\lambda_i)}}$ as shown in \citep{yang2022new}, where $\alpha_i = \left<h, p_i \right>$.

Now, the ratio of these similarities between consecutive layers is,
\small{
\begin{align*}
    \frac{| cos(\left<C^{l+1}h, p_n \right>) |}{| cos(\left<C^{l}h, p_n \right> |} &= \frac{\alpha_n G^{l+1}(\lambda_n)}{\sqrt{\sum_{i=1}^{N} \alpha_i^2 G^{2(l+1)}(\lambda_i)}} \frac{\sqrt{\sum_{i=1}^{N} \alpha_i^2 G^{2l}(\lambda_i)}}{\alpha_n G^{l}(\lambda_n)} \\
    &= G(\lambda_n) \sqrt{\frac{\sum_{i=1}^{N} \alpha_i^2 G^{2l}(\lambda_i)}{\sum_{i=1}^{N} \alpha_i^2 G^{2(l+1)}(\lambda_i)}} \\
\end{align*}
}
In the limit of the layers tending to $\infty$,
\begin{align*}
    &\underset{l \xrightarrow[]{} \infty}{lim} \frac{| cos(\left<C^{l+1}h, p_n \right>) |}{| cos(\left<C^{l}h, p_n \right> |} = \underset{l \xrightarrow[]{} \infty}{lim} G(\lambda_n) \sqrt{\frac{\sum_{i=1}^{N} \alpha_i^2 G^{2l}(\lambda_i)}{\sum_{i=1}^{N} \alpha_i^2 G^{2(l+1)}(\lambda_i)}} \\
    &= G(\lambda_n) \sqrt{ \frac{G^{2l}(\lambda_{max}) \sum_{i=1}^{N} \alpha_i^2 \underset{l \xrightarrow[]{} \infty}{lim} \left( \frac{G^{2l}(\lambda_i)}{G^{2l}(\lambda_{max})}\right)}{\sum_{i=1}^{N} G^{2(l+1)}(\lambda_{max}) \alpha_i^2 \underset{l \xrightarrow[]{} \infty}{lim} \left( \frac{G^{2(l+1)}(\lambda_i)}{G^{2(l+1)}(\lambda_{max})} \right) }} \\
    &= G(\lambda_n) \sqrt{ \frac{\alpha_{max}^2 G^{2l}(\lambda_{max})}{\alpha_{max}^2 G^{2(l+1)}(\lambda_{max})} } \\
    &= \frac{G(\lambda_n)}{G(\lambda_{max})} \\
\end{align*}
\end{proof}

\subsection{Discussion on computational complexity of Spectral Module}
Learning the filter coefficients $f_c^s$ at a given timestep, we can obtain the wavelet operator $g(L)$ at scale $s=1$ using equation $g(s_j x) = \sum_{k=0}^{\infty} c_{j,k} \overline{T_k}(x)$.
The cost for computing the terms $T_k(L) f$ for a sparse graph would be $\mathcal{O}(|E|)$. There would be $M$ such terms in the polynomial function. Also these terms could be reused across different scales. Also the computation of the weighted sum of these terms by the coefficients at scale $s_j$ would incur a complexity of $\mathcal{O}(N \times M_j)$. Thus for $J$ scales the total computational complexity of approximating the wavelet coefficients would be $\mathcal{O}(|E| + N\sum_{j=0}^{J} M_j)$. If the graph is of a bounded degree($d_G$) as is the case with most real world graphs this cost would be reduced to $\mathcal{O}(d_G N + N\sum_{j=0}^{J} M_j)$, at a given snapshot of the graph at time $t$.

\end{document}